\documentclass{article}

\usepackage{fullpage}
\usepackage{amsmath,amsthm,amsfonts,amssymb,mathrsfs}
\usepackage{wrapfig}
\usepackage{algorithm}
\usepackage{algorithmic}
\usepackage{subfigure} 
\usepackage{url}
\usepackage{natbib}
\usepackage{graphicx}

\newcommand{\bx}{\boldsymbol{x}}

\newcommand{\bu}{\boldsymbol{u}}
\newcommand{\btildeu}{\tilde{\boldsymbol{u}}}
\newcommand{\by}{\boldsymbol{y}}

\newcommand{\bloss}{\boldsymbol{\ell}}

\newcommand{\bz}{\boldsymbol{z}}

\newcommand{\bw}{\boldsymbol{w}}

\newcommand{\bv}{\boldsymbol{v}}
\newcommand{\bzero}{\boldsymbol{0}}

\newcommand{\scM}{\mathcal{M}}
\newcommand{\scU}{\mathcal{U}}

\DeclareMathOperator*{\argmin}{argmin}
\DeclareMathOperator*{\argmax}{argmax}

\newcommand{\field}[1]{\mathbb{#1}}

\newcommand{\fX}{\field{X}}

\newcommand{\R}{\field{R}}

\newcommand{\btheta}{\boldsymbol{\theta}}

\newcommand{\theset}[2]{ \left\{ {#1} \,:\, {#2} \right\} }
\newcommand{\inn}[2]{ \left\langle {#1} \,,\, {#2} \right\rangle }
\newcommand{\inner}[1]{ \langle {#1} \rangle }

\newcommand{\Ind}[1]{ \field{I}_{\{{#1}\}} }
\newcommand{\norm}[1]{\left\|{#1}\right\|}
\newcommand{\diag}[1]{\mbox{\rm diag}\!\left\{{#1}\right\}}

\renewcommand{\ss}{\subseteq}

\newcommand{\spin}{\{-1,+1\}}

\newtheorem{lemma}{Lemma}
\newtheorem{theorem}{Theorem}
\newtheorem{cor}{Corollary}

\newcommand{\sign}{{\rm sign}}

\newcommand{\bb}{\boldsymbol{b}}
\newcommand{\paren}[1]{\left({#1}\right)}

\newcommand{\normt}[1]{\norm{#1}_{t}}
\newcommand{\dualnormt}[1]{\norm{#1}_{t,*}}

\begin{document}

\title{
A Generalized Online Mirror Descent \\ with Applications to Classification and Regression
}

\author{Francesco Orabona\\
       Toyota Technological Institute at Chicago\\
       60637 Chicago, IL, USA\\
       \texttt{francesco@orabona.com}
       \and
       Koby Crammer\\
       Department of Electrical Enginering\\
       The Technion\\
       Haifa, 32000 Israel\\
       \texttt{koby@ee.technion.ac.il}
       \and
       Nicol\`o Cesa-Bianchi\\
       Department of Computer Science\\
       Universit\`{a} degli Studi di Milano\\
       Milano, 20135 Italy\\
       \texttt{nicolo.cesa-bianchi@unimi.it}}

\maketitle

\begin{abstract}
Online learning algorithms are fast, memory-efficient, easy to
implement, and applicable to many prediction problems, including classification,
regression, and ranking. Several online algorithms were proposed in the past few
decades, some based on additive updates, like the Perceptron, and some on multiplicative updates, like Winnow.
A unifying perspective on the design and the analysis of online algorithms is provided by online mirror descent, a general prediction strategy from which most first-order algorithms can be obtained as special cases.

We generalize online mirror descent to time-varying regularizers with generic updates. Unlike standard mirror descent, our more general formulation also captures second order algorithms, algorithms for composite losses and algorithms for adaptive filtering. Moreover, we recover, and sometimes improve, known regret bounds as special cases of our analysis using specific regularizers. Finally, we show the power of our approach by deriving a new second order algorithm with a regret bound invariant with respect to arbitrary rescalings of individual features.
\end{abstract}

\section{Introduction}
\label{sec-intro}
Online learning provides a scalable and flexible approach to the solution of a wide range of prediction problems, including classification, regression, and ranking. Popular online strategies for classification and regression include first-order gradient-based methods, such as the standard Perceptron algorithm and its many variants (e.g., $p$-norm Perceptron \citep{Gentile03} and Passive-Aggressive \citep{CrammerDKSSS06}), the Winnow algorithm of \cite{Lit88}, the Widrow-Hoff rule, the Exponentiated Gradient algorithm of \cite{KW97}, and many others.
A more sophisticated adaptation to the data sequence is achieved by second-order methods, which use the inverse of the empirical feature correlation matrix (or, more generally, the inverse of the Hessian) as an dynamic conditioner for the gradient step. These methods include the Vovk-Azoury-Warmuth algorithm \citep{Vov01,AW01} (see also \citep{Forster}), the second-order Perceptron \citep{Cesa-BianchiCG05}, the CW/AROW algorithms \citep{CrammerDP08a,CrammerDP08,CrammerKD09,Crammer:2012:CLC:2343676.2343704},
the adaptive gradient algorithms proposed by \cite{DuchiHS10} and \cite{McMahanS10}, and the Online Newton Step algorithm of \cite{HazanAK07} for exp-concave losses.

Recently, online convex optimization has been proposed as a common unifying framework for designing and analyzing online algorithms. In particular, online mirror descent (OMD) is a generalized online gradient descent algorithm in which the gradient step is mediated by a strongly convex regularization function. By appropriately choosing the regularizer, most first-order online learning algorithms are recovered as special cases of OMD. Moreover, performance guarantees can be derived simply by instantiating the general OMD bounds to the specific regularizer being used. Note that OMD is the online version of the Mirror Descent algorithm for standard convex optimization (over a fixed convex function) by \cite{nemirovskyproblem} ---see also \citep{beck2003mirror} and~\citep[Chapter~11]{Cesa-BianchiL06}. However, before this connection was made explicit, specific instances of OMD had been discovered and analyzed by \cite{WJ97} and \cite{KW01}. These works also pioneered the use of Bregman divergences (ways of measuring distances in linear spaces through convex functions) in the analysis of online algorithms, an approach later extended by \cite{GentileW98} to non-differentiable loss functions.

A series of recent works \citep{Shalev-ShwartzS07,Shalev-Shwartz07,shalev2009mind} investigates a different approach to online analysis based on primal-dual optimization methods. The work of \cite{kakade2012regularization} showed indeed that many instances of OMD can be easily analyzed using only a few basic convex duality properties ---see the recent survey by \cite{Shalev12} for a lucid description of these advances. A related algorithm is Follow the Regularized Leader (FTRL), introduced in \citep{AbernethyHR08,AbernethyHR12},\footnote{There is no clear agreement in the community on the names used to indicate the different algorithms. \cite{hazan2011} distinguishes between an active and a lazy version of OMD, with the former corresponding to online learning through Bregman divergences, and the latter being equivalent to FTRL for linear losses. \cite{Xiao10} rediscovered the lazy version of OMD with time-varying regularizers, and ---following the optimization terminology--- named it Regularized Dual Averaging. \cite{McMahanS10} denote any algorithm minimizing a (linearized) sum of the past losses plus a regularization term as FTRL. Here we follow \cite{Shalev12}, viewing OMD as FTRL applied to linearized losses.}
where at each step the prediction is computed as the minimizer of a regularization term plus the sum of losses on all past rounds. When losses are linear, FTRL and OMD are easily seen to be equivalent \citep{AbernethyHR08,RakhlinT08,hazan2011}.

In this paper, we extend the theoretical framework of \cite{kakade2012regularization} by allowing OMD to use time-varying regularizers with updates that do not necessarily use the sub-gradient of the loss function. Time-varying regularizers have been recently proved to be the key to obtain improved guarantees in unconstrained settings~\citep{StreeterM12,Orabona13,McMahanO14}, and updates that do not use the sub-gradient of the loss are known to lead to second-order algorithms. Indeed, we show that the Vovk-Azoury-Warmuth algorithm, the second-order Perceptron, and the AROW algorithm are recovered as special cases of our generalized OMD. Our unified analysis is simple, and in certain cases achieves slightly improved theoretical guarantees. Our generalized OMD also captures the efficient variants of second-order classification algorithms that only use the diagonal elements of the feature correlation matrix, a result which was not within reach of previous techniques. Moreover, we also show that a proper choice of the time-varying regularizer allows to cope with the composite setting of \citep{DuchiS09,Xiao10,DuchiSST10} without using ad-hoc proof techniques.

Our framework improves on previous results even in the case of first-order algorithms. For example, although aggressive algorithms for binary classification often exhibit a better empirical performance than their conservative counterparts, a theoretical explanation of this behavior remained elusive until now. Using our refined analysis, we are able to prove the first bound for Passive-Aggressive (PA-I) that is never worse (and sometimes better) than the Perceptron bound.

Time-varying regularizers can also be used to perform other types of adaptation to the sequence of observed data. We give a concrete example by introducing new adaptive regularizers corresponding to weighted variants of the standard $p$-norm regularizer. The resulting updates, and their associated regret bounds, enjoy the property of being invariant with respect to arbitrary rescalings of individual features. Moreover, if the best linear predictor for the loss sequence is sparse, then our analysis delivers a better bound than that of OMD with $1$-norm regularization, which is the standard regularizer for the sparse target assumption.

\section{Online convex optimization}
\label{sec-def}
Let $\fX$ be any finite-dimensional linear space equipped with inner product $\inn{\cdot}{\cdot}$. For example, $\fX = \R^d$ where $\inn{\cdot}{\cdot}$ is the vector dot product, or $\fX = \scM_{m,n}$, the space of $m \times n$ real matrices with inner product $\inn{A}{B} = \mathrm{tr}\bigl(A^{\top}B\bigr)$.

In the online convex optimization protocol, an algorithm sequentially chooses elements from a convex set $S \subseteq \fX$, each time incurring a certain loss. At each step $t=1,2,\dots$ the algorithm chooses $\bw_t \in S$ and then observes a convex loss function $\ell_t : S\to\R$. The value $\ell_t(\bw_t)$ is the loss of the learner at step $t$, and the goal is to control the regret,
\begin{equation}
\label{eq:regret}
    R_T(\bu) = \sum_{t=1}^T \ell_t(\bw_t) - \sum_{t=1}^T \ell_t(\bu)
\end{equation}
for all $\bu \in S$ and for any sequence of convex loss functions $\ell_t$. Important application domains for this protocol are online
linear regression and classification. In these settings there is a fixed and known loss function $\ell:\R\times\R\to\R$ and a fixed but unknown sequence $(\bx_1,y_1),(\bx_2,y_2),\dots$ of examples $(\bx_t,y_t)\in\fX\times\R$. At each step $t=1,2,\dots$ the learner observes $\bx_t$ and picks $\bw_t\in S\subseteq\fX$. The loss suffered at step $t$ is then defined as $\ell_t(\bw_t) = \ell\bigl(\inner{\bw,\bx_t},y_t\bigr)$. For example, in regression tasks we might use the square loss $\ell\bigl(\inner{\bw,\bx_t},y_t\bigr) = \bigl(\inner{\bw,\bx_t}-y_t\bigr)^2$. In binary classification, where $y_t\in\{-1,+1\}$, a popular loss function is the hinge loss $\bigl[1-y_t\inner{\bw,\bx_t}\bigr]_+$, where $[a]_+=\max\{0,a\}$. This loss is a convex upper bound on the mistake indicator function $\Ind{y_t\inner{\bw,\bx_t} \le 0}$, which is the true quantity of interest.

Note that Mirror Descent, and its online version OMD, can be also defined and analyzed in spaces much more general than those considered here ---see, e.g., \citep{sridharan2010convex,srebro2011universality}.

\subsection{Further notation and definitions}
We now introduce the basic notions of convex analysis that are used in the paper ---see, e.g., \cite{BauschkeC11}. We consider functions $f : \fX\to\R$ that are closed and convex. This is equivalent to saying that their epigraph $\theset{(\bx,y)}{f(\bx) \le y}$ is a convex and closed subset of $\fX\times\R$. The (effective) domain of $f$, defined by $\theset{\bx\in\fX}{f(\bx) < \infty}$, is a convex set whenever $f$ is convex. We can always choose any $S\subseteq\fX$ as the domain of $f$ by letting $f(\bx) = \infty$ for $\bx\not\in S$.

Given a closed and convex function $f$ with domain $S \ss \fX$, its Fenchel conjugate $f^* : \fX \to \R$ is defined by $f^*(\bu) = \sup_{\bv\in S} \bigl( \langle \bv , \bu \rangle - f (\bv)\bigr)$. Note that the domain of $f^*$ is always $\fX$. Moreover, one can prove that $f^{**}=f$.

A generic norm of a vector $\bu \in \fX$ is denoted by $\norm{\bu}$. Its dual $\norm{\cdot}_*$ is the norm defined by $\norm{\bv}_* = \sup_{\bu}\theset{\langle \bu, \bv \rangle}{ \norm{\bu} \le 1}$. The Fenchel-Young inequality states that $f(\bu)+f^*(\bv) \ge \inner{\bu,\bv}$ for all $\bv,\bu$.

A vector $\bx$ is a subgradient of a convex function $f$ at $\bv$ if $f(\bu) - f (\bv) \ge \langle \bu - \bv, \bx \rangle$ for any $\bu$ in the domain of $f$. The differential set of $f$ at $\bv$, denoted by $\partial f(\bv)$, is the set of all the subgradients of $f$ at $\bv$. If $f$ is also differentiable at $\bv$, then $\partial f(\bv)$ contains a single vector, denoted by $\nabla f(\bv)$, which is the gradient of $f$ at $\bv$. A consequence of the Fenchel-Young inequality is the following: for all $\bx\in\partial f(\bv)$ we have that $f(\bv) + f^*(\bx) = \langle \bv , \bx \rangle$.

A function $f$ is $\beta$-strongly convex with respect to a norm $\norm{\cdot}$ if for any $\bu,\bv$ in its domain, and any $\bx\in\partial f(\bu)$,
\[
    f(\bv) \ge f(\bu) + \langle \bx , \bv - \bu \rangle + \frac{\beta}{2}\norm{\bu-\bv}^2~.
\]
The Fenchel conjugate $f^*$ of a $\beta$-strongly convex function $f$ is everywhere differentiable and $\tfrac{1}{\beta}$-strongly smooth. This means that, for all $\bu,\bv\in\fX$,
\[
    f^*(\bv) \le f^*(\bu) + \langle \nabla f^*(\bu) , \bv - \bu \rangle + \frac{1}{2\beta}\norm{\bu-\bv}_*^2 ~.
\]
See also the paper of \cite{kakade2012regularization} and references therein. A further property of strongly convex functions $f : S \to \R$ is the following: for all $\bu\in\fX$,
\begin{equation}
\label{eq:proj-prop}
    \nabla f^*(\bu) = \argmax_{\bv\in S}\Bigl(\bigl\langle \bv, \bu \bigr\rangle - f(\bv)\Bigr)~.
\end{equation}
This implies the useful identity
\begin{equation}
\label{eq:useful}
    f\bigl(\nabla f^*(\bu)\bigr) + f^*(\bu) = \bigl\langle \nabla f^*(\bu), \bu \bigr\rangle~.
\end{equation}
Strong convexity and strong smoothness are key properties in the design of online learning algorithms.

\section{Online Mirror Descent}
\label{sec-lemma}
We now introduce our main algorithmic tool: a generalization of the OMD algorithm in which the regularizer may change over time. The standard OMD algorithm for online convex optimization ---see, e.g., \citep{Shalev12}--- sets $\bw_t = \nabla f^*(\btheta_t)$, where $f^*$ is a strongly convex regularizer and $\btheta_t$ is updated using subgradient descent: $\btheta_{t+1} = \btheta_t - \eta\bloss'_t$ for $\eta > 0$ and $\bloss'_t \in \partial\ell_t(\bw_t)$. For instance, if $f = \tfrac{1}{2}\norm{\cdot}_2^2$, then $f^* = f$ and OMD specializes to standard online subgradient descent.

We generalize OMD in two ways: first, we allow $f$ to change over time; second, we do not necessarily use the subgradient of the loss to update $\btheta_t$, but rather use an input sequence of generic elements $\bz_t$. The resulting algorithm is summarized in Alg.~\ref{alg:online}.
 
\begin{algorithm}[h]
\begin{algorithmic}[1]
{
\STATE{\bfseries Parameters:} A sequence of strongly convex functions $f_1,f_2,\dots$ defined on a common convex domain $S\ss\fX$.
\STATE{\bfseries Initialize:} $\btheta_1=\boldsymbol{0} \in \fX$
\FOR{$t=1,2,\dots$}
\STATE{Choose $\bw_t=\nabla f_t^*(\btheta_t)$}
\STATE{Observe $\bz_t \in \fX$}
\STATE{Update $\btheta_{t+1}=\btheta_{t} + \bz_t$}
\ENDFOR
}
\end{algorithmic}
\caption{Online Mirror Descent}
\label{alg:online}
\end{algorithm}
Note the following remarkable property: while $\btheta_t$ follows an arbitrary trajectory in $\fX$, as determined by the input sequence $\bz_t$, because of~(\ref{eq:proj-prop}) the property $\bw_t\in S$ holds uniformly over $t$. 

OMD with specific time-varying regularizers was implicitly analyzed by \cite{Vov01} and independently by \cite{AW01}, who introduced the
proof ideas currently used here. Another early example of OMD with time-varying regularizers is due to \cite{BartlettHR07}. A more explicit analysis is contained in the work of \cite{SridharanT10}. The following lemma is a generalization of two corollaries of \citep[Corollary~4]{kakade2012regularization} and of \citep[Corollary~3]{DuchiHS10}.
\begin{lemma}
\label{lemma:general}
Assume OMD is run with functions $f_1,f_2,\dots, f_T$ defined on a common convex domain $S\ss\fX$ and such that each $f_t$ is $\beta_t$-strongly convex with respect to the norm $\normt{\cdot}$. Let $\dualnormt{\cdot}$ be the dual norm of $\normt{\cdot}$, for $t=1,2,\dots,T$.
Then, for any $\bu \in S$,
\begin{equation*}
\sum_{t=1}^T \langle \bz_t , \bu- \bw_t \rangle \leq f_T(\bu) + \sum_{t=1}^{T} \left(\frac{\dualnormt{\bz_t}^2}{2 \beta_t} + f^*_t(\btheta_t)-f^*_{t-1}(\btheta_{t}) \right)
\end{equation*}
where we set $f^*_0(\boldsymbol{0})=0$.
Moreover, for all $t \ge 1$, we have
\begin{equation}
\label{eq:residue}
f^*_t(\btheta_t)-f^*_{t-1}(\btheta_{t})\leq f_{t-1}(\bw_t)-f_t(\bw_t)~.
\end{equation}
\end{lemma}
\begin{proof}
Let $\Delta_t=f^*_t(\btheta_{t+1})-f^*_{t-1}(\btheta_{t})$. Then
\begin{equation*}
\sum_{t=1}^{T} \Delta_t= f^*_T(\btheta_{T+1})-f^*_0(\btheta_{1}) = f^*_T(\btheta_{T+1})~.
\end{equation*}
Since the functions $f_t^*$ are $\tfrac{1}{\beta_t}$-strongly smooth with respect to $\dualnormt{\cdot}$, and recalling that $\btheta_{t+1}=\btheta_{t} + \bz_t$,
\begin{align*}
    \Delta_t 
&=
    f^*_t(\btheta_{t+1})-f^*_t(\btheta_t)+f^*_t(\btheta_t)-f^*_{t-1}(\btheta_{t})
\\& \le
    f^*_t(\btheta_t)-f^*_{t-1}(\btheta_{t}) + \langle \nabla f^*_t (\btheta_t), \bz_t \rangle + \frac{1}{2 \beta_t} \dualnormt{\bz_t}^2
\\& =
    f^*_t(\btheta_t)-f^*_{t-1}(\btheta_{t}) + \langle \bw_t, \bz_t \rangle + \frac{1}{2 \beta_t} \dualnormt{\bz_t}^2
\end{align*}
where we used the definition of $\bw_t$ in the last step.
On the other hand, the Fenchel-Young inequality implies
\begin{equation*}
    \sum_{t=1}^{T} \Delta_t = f^*_{T} (\btheta_{T+1})
\ge
    \langle \bu, \btheta_{T+1} \rangle - f_T(\bu)
=
    \sum_{t=1}^{T} \langle \bu, \bz_{t} \rangle - f_T(\bu)~.
\end{equation*}
Combining the upper and lower bound on $\Delta_t$ and summing over $t$ we get
\begin{align*}
    \sum_{t=1}^{T} \langle \bu, \bz_{t} \rangle - f_T(\bu)
\le
    \sum_{t=1}^{T} \Delta_t
\le
    \sum_{t=1}^{T} \left(f^*_t(\btheta_t)-f^*_{t-1}(\btheta_{t}) + \langle \bw_t, \bz_t \rangle + \frac{1}{2 \beta_t} \dualnormt{\bz_t}^2 \right)~.
\end{align*}
We now prove the second statement. Recalling again the definition of $\bw_t$, we have that~(\ref{eq:useful}) implies 
$
f^*_t(\btheta_t) = \langle \bw_t ,\btheta_t \rangle- f_t(\bw_t)
$.
On the other hand, the Fenchel-Young inequality implies that
$
-f^*_{t-1}(\btheta_t) \leq f_{t-1}(\bw_t) - \langle \bw_t , \btheta_t \rangle
$.
Combining the two we get
$f^*_t(\btheta_t)-f^*_{t-1}(\btheta_{t}) \leq f_{t-1}(\bw_t) - f_t(\bw_t)
$,
as desired.
\end{proof}
We are now ready to prove regret bounds for OMD applied to three different classes of time-varying regularizers.
\begin{cor}
\label{cor:convex_loss}
Let $S$ a convex set, $F : S \to \R$ be a convex function, and let $g_1,g_2,\dots,g_T$ be a sequence of convex functions $g_t : S \to \R$ such that $g_{t}(\bu) \le g_{t+1}(\bu)$ for all $t=1,2,\dots,T$ and all $\bu \in S$. Fix $\eta > 0$ and assume $f_t = g_t + \eta t F$ are $\beta_t$-strongly convex w.r.t. $\norm{\cdot}_{f_t}$. For each $t=1,2,\dots,T$ let $\dualnormt{\cdot}$ be the dual norm of $\normt{\cdot}$.
If OMD is run on the input sequence $\bz_t= -\eta\bloss'_t$ for some $\bloss'_t\in\partial\ell_t(\bw_t)$, then
\begin{align}
\label{eq:composite_bound}
    \sum_{t=1}^{T} \bigl(\ell_t(\bw_t) + F(\bw_t)\bigr) - \sum_{t=1}^{T} \bigl(\ell_t(\bu) + F(\bu)\bigr)
\le
    \frac{g_T(\bu)}{\eta} + \eta \sum_{t=1}^{T} \frac{\dualnormt{\bloss'_t}^2}{2 \beta_t}
\end{align}
for all $\bu\in S$.

Moreover, if $f_t = g\sqrt{t} + \eta t F$ where $g : S \to \R$ is $\beta$-strongly convex w.r.t. $\norm{\cdot}$, then
\begin{align}
\label{eq:tuned-bound}
    \sum_{t=1}^{T} \bigl(\ell_t(\bw_t) + F(\bw_t)\bigr) - \sum_{t=1}^{T} \bigl(\ell_t(\bu) + F(\bu)\bigr)
\le
    \sqrt{T} \left( \frac{g(\bu)}{\eta} + \frac{\eta}{\beta} \max_{t \le T} \norm{\bloss'_t}_*^2 \right)
\end{align}
for all $\bu\in S$.

Finally, if $f_t = t\,F$, where $F$ is $\beta$-strongly convex w.r.t. $\norm{\cdot}$, then
\begin{align}
\label{eq:log-bound}
    \sum_{t=1}^{T} \bigl(\ell_t(\bw_t) + F(\bw_t)\bigr) - \sum_{t=1}^{T} \bigl(\ell_t(\bu) + F(\bu)\bigr)
\le
    \max_{t \le T} \norm{\bloss'_t}_*^2 \frac{(1+\ln T)}{2\beta}
\end{align}
for all $\bu\in S$.
\end{cor}
\begin{proof}
By convexity, $\ell_t(\bw_t) - \ell_t(\bu) \leq \tfrac{1}{\eta} \langle \bz_t, \bu - \bw_t\rangle$. Using Lemma~\ref{lemma:general} we have,
\begin{align*}
    \sum_{t=1}^{T} &\langle \bz_t, \bu - \bw_t\rangle
\\ &\le
    g_T(\bu) + \eta T F(\bu) + \eta^2 \sum_{t=1}^{T} \frac{\dualnormt{\bloss'_t}^2}{2 \beta_t} + \eta \sum_{t=1}^{T} \bigl((t-1) F(\bw_t) - t F(\bw_{t})\bigr)
\end{align*}
where we used the fact that the terms $g_{t-1}(\bw_t) - g_{t}(\bw_{t})$ are nonpositive as per our assumption.
Reordering terms we obtain~(\ref{eq:composite_bound}). In order to obtain~(\ref
{eq:tuned-bound}) it is sufficient to note that, by definition of strong convexity, $g\sqrt{t}$ is $\beta\sqrt{t}$-strongly convex because $g$ is $\beta$-strongly convex, hence $f_t$ is $\beta\sqrt{t}$-strongly convex too. 
The elementary inequality
$
    \sum_{t=1}^T \frac{1}{\sqrt{t}} \le 2\sqrt{T}
$
concludes the proof of~(\ref{eq:tuned-bound}). Finally, bound~(\ref{eq:log-bound}) is proven by observing that $f_t = t\,F$ is $\beta t$-strongly convex because $F$ is $\beta$-strongly convex.
The elementary inequality $\sum_{t=1}^T \tfrac{1}{t} \le 1 + \ln T$ concludes the proof.
\end{proof}
Note that the regret bounds obtained in Corollary~\ref{cor:convex_loss} are for the composite setting, where the algorithm minimizes the sum $\ell_t(\cdot) + F(\cdot)$ of two functions, where the first one is typically a loss and the other is a regularization term. Here, the only hypothesis on $F$ is convexity, hence $F$ can be a nondifferentiable function, like $\norm{\cdot}_1$ for inducing sparse solutions. While the composite setting is considered more difficult than the standard one, and requires specific ad-hoc algorithms~\citep{DuchiS09,Xiao10,DuchiSST10}, here we show that this setting can be efficiently solved using OMD with a specific choice of the time-varying regularizer. Thus, we recover the results about minimization of strongly convex and composite loss functions, and adaptive learning rates, in a simple unified framework. Note that~\eqref{eq:residue}, which was missing in previous analyses of OMD, is the key result to obtain this level of generality. 

A special case of OMD is the Regularized Dual Averaging framework of \cite{Xiao10}, where the prediction at each step is defined by
\begin{align}
\bw_t = \argmin_{\bw} \frac{1}{t-1} \sum_{s=1}^{t-1} \bw^{\top}\bloss'_s + \frac{\beta_{t-1}}{t-1} g(\bw) + F(\bw)
\end{align}
for some $\bloss'_s\in\partial\ell_s(\bw_s)$, $s=1,\dots,t-1$.
Using~(\ref{eq:proj-prop}), it is easy to see that this update is equivalent\footnote{Although \cite{Xiao10} explicitly mentions that his results cannot be recovered with the primal-dual proofs, here we prove the contrary.} to
$\bw_t = \nabla f_t^*\left(\sum_{s=1}^{t-1} \bloss'_s\right)$,
where $f_t(\bw)=\beta_{t-1}\,g(\bw) + (t-1)\,F(\bw)$. The framework of \cite{Xiao10} has been extended by \cite{DuchiSST10} to allow the strongly convex part of the regularizer to increase over time. A bound similar to (\ref{eq:composite_bound}) has been also recently presented by \cite{DuchiSST10}. There, a more immediate trade-off between the current gradient and the Bregman divergence from the new solution to the previous one is used to update at each time step.
However, in both cases their analysis is not flexible enough to include algorithms whose update does not use the sub-gradient of the loss function. Examples of such algorithms are the Vovk-Azoury-Warmuth algorithm of the next section and the online binary classification algorithms of Section~\ref{sec-class}.

\section{Online regression with square loss}
In this section we apply Lemma~\ref{lemma:general} to recover known regret bounds for online regression and adaptive filtering with the square loss. For simplicity, we set $\fX=\R^d$ and let the inner product $\inner{\bu,\bx}$ be the standard dot product $\bu^{\top}\bx$. We also set $\ell_t(\bu) = \tfrac{1}{2}\bigl(y_t - \bu^{\top}\bx_t\bigr)^2$ where $(\bx_1,y_1),(\bx_2,y_2),\dots$ is some arbitrary sequence of examples $(\bx_t,y_t)\in\R^d\times\R$.

First, note that it is possible to specialize OMD to the Vovk-Azoury-Warmuth algorithm for online regression. Remember that, at each time step $t$, the Vovk-Azoury-Warmuth algorithm predicts with
\begin{align*}
\bw_t 
&= \argmin_{\bw} \ \frac{a}{2}\|\bw\|^2 + \frac{1}{2}\sum_{s=1}^{t-1} \bigl(y_s - \bw^{\top}\bx_s\bigr)^2 + \frac{1}{2}\bigl(\bw^{\top}\bx_t\bigr)^2 \\
&= \argmin_{\bw} \ \frac{a}{2} \|\bw\|^2  + \frac{1}{2} \sum_{s=1}^{t} (\bw^{\top}\bx_s)^2 - \sum_{s=1}^{t-1} y_s \bw^{\top}\bx_s  \\
&= \argmin_{\bw} \ \frac{1}{2} \bw^\top \left(a I + \sum_{i=1}^{t} \bx_s \bx_s^\top\right) \bw - \sum_{s=1}^{t-1} y_s \bw^{\top}\bx_s  \\
&= \left(a I + \sum_{s=1}^{t} \bx_s \bx_s^\top\right)^{-1} \sum_{i=1}^{t-1} y_s \bx_s~.
\end{align*}
Now, by letting $A_0=a\,I_d$, $A_t=A_{t-1} + \bx_t \bx_t^\top$ for $t \geq 1$, and $\bz_s = y_s \bx_s$, we obtain the OMD update
$
    \bw_t = A_t^{-1} \btheta_t = \nabla f_t^*(\btheta_t)
$,
where $f_t(\bu)=\tfrac{1}{2}\bu^\top A_t \bu$ and $f_t^*(\btheta) = \tfrac{1}{2}\btheta^\top A^{-1}_t \btheta$. Note that $\bz_t$ is \emph{not} equal to the negative gradient of the square loss. In fact, the special structure of the square loss allows us to move some of the terms inside the regularizer, and use proxies for the gradients of the losses.
The regret bound of this algorithm ---see, e.g., \cite[Theorem~11.8]{Cesa-BianchiL06}--- is recovered from Lemma~\ref{lemma:general} by noting that $f_t$ is $1$-strongly convex with respect to the norm $\normt{\bu} = \sqrt{\bu^\top A_t \bu}$. Hence, the regret~(\ref{eq:regret}) is controlled as follows
\begin{align*}
R_T(\bu)
&= \frac{1}{2}\sum_{t=1}^T (y_t - \bw_t^\top \bx_t)^2 - \frac{1}{2}\sum_{t=1}^T (y_t - \bu^\top \bx_t)^2 \\
&= \sum_{t=1}^T (y_t \bu^\top \bx_t - y_t \bw_t^\top \bx_t) - f_T(\bu) + \frac{a}{2}\norm{\bu}^2 + \frac{1}{2} \sum_{t=1}^T (\bw^\top_t \bx_t)^2
\\ &\le
    f_T(\bu) + \sum_{t=1}^{T} \left(\frac{y_t^2\dualnormt{\bx_t}^2}{2} + f^*_t(\btheta_t)-f^*_{t-1}(\btheta_{t}) \right) - f_T(\bu) + \frac{a}{2}\norm{\bu}^2 \\
&\quad + \frac{1}{2} \sum_{t=1}^T (\bw^\top_t \bx_t)^2
\\ &\le
    \frac{a}{2}\norm{\bu}^2 + \frac{Y^2}{2}\sum_{t=1}^{T} \bx^\top_t A^{-1}_{t} \bx_t
\end{align*}
since
$
    f^*_t(\btheta_t)-f^*_{t-1}(\btheta_{t}) \le f_{t-1}(\bw_t) - f_t(\bw_t) = - \tfrac{1}{2}(\bw^\top_t \bx_t)^2
$,
and where $Y = \max_t |y_t|$.
%

A related setting is that of adaptive filtering ---see, e.g., \citep{KivinenWH06} and references therein. In this setting the output signals $y_t$ are given by
$
  y_t = \bu^\top \bx_t + \nu_t
$,
where $\bu$ is unknown and $\nu_t$ is some arbitrary noise process. The goal is to recover the uncorrupted output $\bu^\top\bx_t$. This can be achieved by minimizing a suitable notion of regret, namely the adaptive filtering regret w.r.t.~the square loss,
\[
    R^{\textsc{af}}_T(\bu) := \sum_{t=1}^T \bigl(\bw_t^\top \bx_t - \bu^\top \bx_t\bigr)^2~.
\]
\cite{KivinenWH06} introduced the $p$-norm LMS algorithm, addressing adaptive filtering problems in a purely online nonostochastic setting. We now show that OMD can be easily applied to online adaptive filtering. First, note that
\begin{align}
\nonumber
    R_T(\bu) +\frac{1}{2}R^{\textsc{af}}_T(\bu)
&=
    \sum_{t=1}^T \Bigl( \bigl(y_t- \bw_t^\top \bx_t\bigr)^2 - \bigl(y_t- \bu^\top \bx_t\bigr)^2 + \frac{1}{2}\bigl(\bw_t^\top \bx_t - \bu^\top \bx_t\bigr)^2 \Bigr) \\
\nonumber
&=
    \sum_{t=1}^T \Bigl( \bigl(y_t- \bw_t^\top \bx_t\bigr) \bu^\top \bx_t - \bigl(y_t- \bw_t^\top \bx_t\bigr) \bw_t^\top \bx_t\Bigr) \\
\label{eq:filtering1}
&=
    \sum_{t=1}^T \bigl(\bu - \bw_t\bigr)^{\top}\bz_t
\end{align}
where we set $\bz_t= - \bigl(y_t- \bw_t^\top \bx_t\bigr) \bx_t$. Now, pick any function $f$ which is $1$-strongly convex with respect to some norm $\norm{\cdot}$ and let $f_t(\bu)= X^2_t f(\bu)$, where $X_t = \max_{s \leq t}\norm{\bx_s}_*$. Lemma~\ref{lemma:general} then immediately implies that
\begin{align}
\label{eq:filtering2}
    \sum_{t=1}^T \bigl(\bu - \bw_t\bigr)^{\top}\bz_t
&\le
    f_T(\bu) + \frac{1}{2}\sum_{t=1}^{T} \bigl(y_t- \bw_t^\top \bx_t\bigr)^2
\end{align}
where we used the $X^2_t$-strong convexity of $f_t$ and the fact the $f_{t} \geq f_{t-1}$.
Combining~(\ref{eq:filtering1}) with~(\ref{eq:filtering2}) and simplifying, we obtain the following adaptive filtering bound
\begin{align*}
    R^{\textsc{af}}_T(\bu)
\leq
    2 X^2_T f(\bu)  + \sum_{t=1}^T \bigl(y_t - \bu^\top \bx_t\bigr)^2~.
\end{align*}
This is a direct generalization to arbitrary regularizers $f$ of the bound by \cite{KivinenWH06}. However, their algorithm requires the prior knowledge of the maximum norm of $\bx_t$ to set a critical parameter. Instead, our algorithm, through the use of an increasing regularizer, has the ability to adapt to the maximum norm of $\bx_t$, without using any prior knowledge.

\section{Scale-invariant algorithms}
\label{sec-new}
In this section we show the full power of our framework by introducing two new scale-invariant algorithms for online linear regression with an arbitrary convex and Lipschitz loss function.

Recall the online linear regression setting: given a convex loss $\ell : \R^2\to\R$ and a fixed but unknown sequence $(\bx_1,y_1),(\bx_2,y_2),\dots$ of examples, at each step $t=1,2,\dots$ the online algorithm observes $\bx_t$ and picks $\bw_t$. The associated loss is $\ell_t(\bw_t) = \ell\bigl(\bw_t^{\top}\bx_t,y_t\bigr)$.

Let $\bu\in\R^d$ be any fixed predictor with small total loss $\ell_1(\bu)+\cdots+\ell_T(\bu)$.
Because of linearity, an arbitrary rescaling of any individual feature, $x_{t,i} \to c\,x_{t,i}$ for $t=1,\dots,T$ (while $y_1,\dots,y_T$ are kept fixed) can be offset by a corresponding rescaling of $u_i$ without affecting the total loss. We might ask a similar scale-invariance for the online predictor. In other words, we would like the online algorithm to be independent of the units in which each data coordinate is expressed.

We now introduce two new time-varying regularizers that achieve this goal. As in the previous section, let $\fX=\R^d$ and let the inner product $\inner{\bu,\bx}$ be the standard dot product $\bu^{\top}\bx$.

The new regularizers are based on the following generalization of the squared $q$-norm.
Given $(a_1,\dots,a_d) \in \R_+$ and $q \in (1,2]$ define the weighted $q$-norm of $\bw \in \R^d$ by
\[
    \left( \sum_{i=1}^d \vert w_i\vert^q a_i \right)^{1/q}~.
\]
Define the corresponding regularization function by
\[
f(\bw)=\frac{1}{2(q-1)}\left( \sum_{i=1}^d \vert w_i\vert^q a_i \right)^{2/q}~.
\]
This function has the following properties (proof in the Appendix).
\begin{lemma}
\label{l:weighted-q}
The Fenchel conjugate of $f$ is
\begin{equation}
f^*(\btheta)=\frac{1}{2(p-1)} \left( \sum_{i=1}^d \vert \theta_i\vert^p \, a_i^{1-p} \right)^{2/p}
\qquad \text{for}\quad p = \frac{q}{q-1}~.
\end{equation}
Moreover, the function $f(\bw)$ is $1$-strictly convex with respect to the norm
\[
\left( \sum_{i=1}^d \vert x_i \vert^q a_i \right)^{1/q}
\]
whose dual norm is defined by
\[
\left( \sum_{i=1}^d \vert \theta_i \vert^p a_i^{1-p} \right)^{1/p}~.
\]
\end{lemma}
Given an arbitrary sequence $(\bx_1,y_1),(\bx_2,y_2),\dots$ of examples, we assume OMD is run with $\bz_t = -\eta\bloss'_t$ where, as usual, $\bloss'_t \in \partial\ell_t(\bw_t)$. We also assume the loss $\ell$ is such that $\ell(\cdot,y)$ is $L$-Lipschitz for each $y\in\R$ and $\ell_t(\bw) = \ell\bigl(\bw^{\top}\bx_t,y_t\bigr)$ is convex for all $t$. 
In the rest of this section, the following notation is used: ${\displaystyle b_{t,i} = \max_{s=1,\dots,t} |x_{s,i}|,\, m_t =\max_{s=1,\dots,t} \norm{\bx_s}_0,\, p_t = 2 \ln m_t }$, and
\[
  \beta_t = \sqrt{ e L^2 (p_t-1) + \sum_{s=1}^{t-1} (p_s-1) \left(\sum_{i=1}^d \left( \frac{|\ell'_{s,i}|}{b_{s,i}} \right)^{p_s} \right)^{2/{p_s}} }~.
\]
The time-varying regularizers we consider are defined as follows,
\begin{align}
\label{eq:f1}
  f_t(\bu) &= \frac{\beta_t}{2}\left( \sum_{i=1}^d \bigl(|u_i| b_{t,i}\bigr)^{q_t} \right)^{2/q_t} \qquad \text{for}\quad q_t = \frac{p_t}{p_t-1}
\\
\label{eq:f2}
  f_t(\bu) &= \frac{\sqrt{d}}{2}\left( \sum_{i=1}^d \bigl(|u_i| b_{t,i}\bigr)^2 \sqrt{L^2+\sum_{s=1}^{t-1} \left(\frac{\ell'_{s,i}}{b_{s,i}}\right)^2} \right)~.
\end{align}
As we show next, regularizers of type~(\ref{eq:f1}) give regret bounds that exhibit a logarithmic dependency on the maximum number of non-zero components observed. Instead, regularizers of type~(\ref{eq:f2}) give bounds that depend on $\sqrt{d}$, and used a different learning rate for each coordinate. Roughly speaking, the first regularizer provides scale-invariance with a logarithmic dependency on the dimension $d$ obtained through the $p$-norm regularization~\citep{Gentile03}. The second regularizer, instead, corresponds to a scale invariant version of AdaGrad~\citep{DuchiHS10}.

In order to spell out the OMD update, we compute the derivative of the Fenchel dual of the regularization functions. Using the fact that if $g(\bw)=a f(\bw)$, then $g^*(\btheta)=a f^*(\frac{\theta}{a})$, for regularizers of type~(\ref{eq:f1}) we have
\begin{align}
\bigl(\nabla f^*_t(\btheta)\bigr)_j =  \frac{1}{\beta_t (p_t-1)} \left( \sum_{i=1}^d \left(\frac{|\theta_i|}{b_{t,i}}\right)^{p_t} \right)^{2/p_t-1} \frac{|\theta_j|^{p_t-1}}{b^{p_t}_{t,j}} \sign(\theta_j)~.
\label{eq:update1}
\end{align}
For regularizers of type~(\ref{eq:f2}) we have
\begin{align}
\bigl(\nabla f^*_t(\btheta)\bigr)_j = \frac{\theta_j}{ b^2_{t,j} \sqrt{d} \sqrt{L^2+\sum_{s=1}^{t-1} \left(\frac{\ell'_{s,j}}{b_{s,j}}\right)^2}}~. \label{eq:update2}
\end{align}
These updates are such that $\bw_t^{\top}\bx_t$ is independent of arbitrary rescalings of individual features. To see this, recall that $\btheta_t = -\eta\bigl(\bloss'_1+\cdots+\bloss'_{t-1}\bigr)$ and 
\[
    \bloss'_s = \left.\frac{\partial\ell(z,y_s)}{\partial z}\right|_{z=\bw_s^{\top}\bx_s}\bx_s
\]
where the partial derivative is at most $L$ by Lipschitzness of $\ell$. Hence, the ratios $|\ell'_{t,j}|\big/b_{t,j}$ and $|\theta_{t,j}|\big/b_{t,j}$ are invariant with respect to arbitrary rescalings of the $j$-th feature. So, in both \eqref{eq:update1} and \eqref{eq:update2}, $w_{t,j}$ scales as $1/b_{t,j}$, and we have that $\bw_t^{\top}\bx_t$ is invariant to the rescaling of individual coordinates.

The formula in the right-hand side of~\eqref{eq:update2} also shows that, similar to the updates studied in \citep{McMahanS10,DuchiHS10}, the second type of regularizer induces a different learning rate for each component of $\bw_t$.

We now prove the following regret bounds.
\begin{theorem}
\label{theo:newbound}
If OMD is run using regularizers of type~(\ref{eq:f1}), then for any $\bu \in \R^d$
\begin{align*}
    R_T(\bu)
\leq
    L \sqrt{e (T+1) (2 \ln m_T -1)}\left( \frac{1}{2 \eta}\left(\sum_{i=1}^d |u_i| b_{T,i} \right)^2 + \eta \right)~.
\end{align*}
If OMD is run using regularizers of type~(\ref{eq:f2}), then for any $\bu \in \R^d$
\begin{align*}
    R_T(\bu)
\le
    L \sqrt{d(T+1)} \left(\frac{1}{2 \eta}\sum_{i=1}^d \bigl(u_i b_{T,i}\bigr)^2 + \eta \right)~.
\end{align*}
\end{theorem}
\begin{proof}
For the first algorithm, note that $m^{{2}/{p_t}}_t = e$, and setting $q_t = \left(1-\frac{1}{p_t}\right)^{-1}$, we have $q_t (1-p_t)=-p_t$. Further note that
$
    f^*_t(\btheta_t)-f^*_{t-1}(\btheta_{t})\leq f_{t-1}(\bw_t)-f_t(\bw_t) \le 0
$,
where $f_{t-1} \le f_t$ because $q_t$ is decreasing, $b_{t,i}$ is increasing, and $\beta_t$ is also increasing. Hence, using the convexity of $\ell_t$ and Lemma~\ref{lemma:general}, we may write
\begin{align}
    R_T(\bu)
&\leq
    \sum_{t=1}^T (\bloss'_t)^{\top}\bigl(\bu-\bw_t\bigr) \notag 
\\ &\leq
    \frac{\beta_T}{2 \eta}\left( \sum_{i=1}^d \bigl(|u_i| b_{T,i}\bigr)^{q_T} \right)^{2/q_T} + \eta \sum_{t=1}^{T} \frac{1}{2 \beta_t (q_t-1)} \left(\sum_{i=1}^{d} \frac{ |\ell'_{t,i}|^{p_t} }{b^{p_t}_{t,i}} \right)^{{2}/{p_t}}~. \label{eq:newbound1}
\end{align}
For the second term in~\eqref{eq:newbound1}, using the fact that $\frac{1}{q_t-1}=p_t-1$ and the $L$-Lipschitzness, we have
\begin{align*}
    \sum_{t=1}^{T} &\frac{1}{2 \beta_t (q_t-1)} \left(\sum_{i=1}^{d} \frac{ |\ell'_{t,i}|^{p_t} }{b^{p_t}_{t,i}} \right)^{{2}/{p_t}}
\\&=
    \frac{1}{2}\sum_{t=1}^{T} \frac{(p_t-1) \left(\sum_{i=1}^{d} \frac{ |\ell'_{t,i}|^{p_t} }{b^{p_t}_{t,i}} \right)^{{2}/{p_t}}}{\sqrt{ L^2 m_t^{{2}/{p_t}} (p_t-1)+\sum_{s=1}^{t-1} (p_s-1)\left(\sum_{i=1}^d \frac{|\ell'_{s,i}|^{p_s}}{b^{p_s}_{s,i}}\right)^{{2}/{p_s}} }}
\\ &\leq
    \frac{1}{2}\sum_{t=1}^{T} \frac{(p_t-1) \left(\sum_{i=1}^{d} \frac{ |\ell_{t,i}|^{p_t} }{b^{p_t}_{t,i}} \right)^{{2}/{p_t}}}{\sqrt{ \sum_{s=1}^{t} (p_s-1)\left(\sum_{i=1}^d \frac{|\ell'_{s,i}|^{p_s}}{b^{p_s}_{s,i}}\right)^{{2}/{p_s}} }}
\\ &\leq
    \sqrt{\sum_{t=1}^{T} (p_t-1) \left(\sum_{i=1}^{d} \frac{ |\ell'_{t,i}|^{p_t} }{b^{p_t}_{t,i}} \right)^\frac{2}{p_t} }
\\ &\leq
    \beta_{T+1}
\end{align*}
where the second inequality uses the elementary inequality
\begin{equation}
\label{eq:auer}
    \sum_{t=1}^T \frac{a_t}{\sqrt{\sum_{s=1}^t a_s}} \le 2\sqrt{\sum_{t=1}^T a_t} \qquad \text{for $a_1,\dots,a_T \ge 0$}
\end{equation}
(see, e.g., \cite[Lemma~3.5]{AuerCG02}). Hence we have
\begin{align*}
R_T(\bu) \leq \beta_{T+1} \left(  \frac{\left(\sum_{i=1}^d |u_i| b_{T,i} \right)^2}{2 \eta} + \eta \right)~.
\end{align*}
Finally, note that
\[
\beta_{T+1} =\sqrt{ e L^2 p_T +\sum_{s=1}^T (p_s-1) \left(\sum_{i=1}^d \left(\frac{|\ell'_{s,i}|}{b_{s,i}}\right)^{p_s}\right)^{{2}/{p_s}} } \le  L \sqrt{e (T+1) (2 \ln m_T -1)}~.
\]
The proof of the second bound is similar. First note that
$
    f^*_t(\btheta_t)-f^*_{t-1}(\btheta_{t})\leq f_{t-1}(\bw_t)-f_t(\bw_t) \le 0
$,
where $f_{t-1} \le f_t$ is easily verified by inspection of~(\ref{eq:f2}). Using the convexity of $\ell_t$ and Lemma~\ref{lemma:general} we then obtain
\begin{align*}
    R_T(\bu)
&\leq
    \sum_{t=1}^T (\bloss'_t)^{\top}\bigl(\bu-\bw_t\bigr) \notag 
\\ &\leq
    \frac{\sqrt{d}}{2 \eta}\left( \sum_{i=1}^d \bigl(u_i b_{T,i}\bigr)^2 \sqrt{L^2+\sum_{s=1}^{T-1} \left(\frac{\ell'_{s,i}}{b_{s,i}}\right)^2} \right) 
\\ &\quad +
    \frac{\eta}{2\sqrt{d}} \sum_{t=1}^{T} \left(\sum_{i=1}^{d} \left(\frac{\ell'_{t,i}}{b_{t,i}}\right)^2 \frac{1}{\sqrt{L^2+\sum_{s=1}^{t-1} \left(\frac{\ell'_{s,i}}{b_{s,i}}\right)^2}} \right)~.
\end{align*}
For the second term, we have
\begin{align*}
    \frac{\eta}{2\sqrt{d}} \sum_{t=1}^{T} \left(\sum_{i=1}^{d} \left(\frac{\ell'_{t,i}}{b_{t,i}}\right)^2 \frac{1}{\sqrt{L^2+\sum_{s=1}^{t-1} \left(\frac{\ell'_{s,i}}{b_{s,i}}\right)^2}} \right)
&\leq
    \frac{\eta}{2\sqrt{d}}\sum_{t=1}^{T} \sum_{i=1}^{d} \frac{ \bigl(\ell'_{t,i}\big/b_{t,i}\bigr)^2 }{\sqrt{\sum_{s=1}^{t} \left(\frac{\ell'_{s,i}}{b_{s,i}}\right)^2 }}
\\ &\leq
    \frac{\eta}{\sqrt{d}} \sum_{i=1}^{d} \sqrt{\sum_{t=1}^{T} \left(\frac{\ell'_{t,i}}{b_{t,i}}\right)^2 }
\end{align*}
where the last inequality uses~(\ref{eq:auer}). The proof is finished by noting that
\begin{align}
\nonumber
    &\frac{\sqrt{d}}{2 \eta}\left( \sum_{i=1}^d \bigl(u_i b_{T,i}\bigr)^2 \sqrt{L^2+\sum_{t=1}^{T-1} \left(\frac{\ell'_{t,i}}{b_{t,i}}\right)^2} \right) + \frac{\eta}{\sqrt{d}} \sum_{i=1}^{d} \sqrt{\sum_{t=1}^{T} \left(\frac{\ell'_{t,i}}{b_{t,i}}\right)^2}
\\ &\quad \le
\label{eq:tocompare}
    \sum_{i=1}^d \left( \sqrt{L^2+\sum_{t=1}^T \left(\frac{\ell'_{t,i}}{b_{t,i}}\right)^2 } \left(\frac{\sqrt{d}}{2\eta} \bigl(u_i b_{T,i}\bigr)^2 + \frac{\eta}{\sqrt{d}} \right)  \right)
\\ &\quad \le
\nonumber
    L \sqrt{d(T+1)} \left(\frac{1}{2 \eta}\sum_{i=1}^d \bigl(u_i b_{T,i}\bigr)^2 + \eta \right)~.
\end{align}
\end{proof}
Note that both bounds are invariant with respect to arbitrary scaling of individual coordinates of the data points $\bx_t$ in the following sense: if the $i$-th feature is rescaled $x_{t,i} \to c\,x_{t,i}$ for all $t$, then a corresponding rescaling $u_i \to u_i/c$, leaves the bounds unchanged.

This invariance property is not exhibited by standard OMD run with non-adaptive regularizers, whose regret bounds are of the form
$
    \norm{\bu}\max_t\norm{\bx_t}_*\sqrt{T}
$.
In particular, by an appropriate tuning of $\eta$ the regret in Corollary~\ref{cor:convex_loss} for the regularizer type~(\ref{eq:f1}) is bounded by a quantity of the order of
\[
    \left(\sum_{i=1}^d |u_i| \max_t |x_{t,i}|\right) \sqrt{T\ln d}~.
\]
When the good $\bu$ are sparse, implying that the norms $\norm{\bu}_1$ are small, this is always better than running standard OMD with a non-weighted $q$-norm regularizer. For $q \to 1$ (the best choice for the sparse $\bu$ case), this gives bounds of the form
\[
    \left(\norm{\bu}_1 \max_t\norm{\bx_t}_{\infty}\right)\sqrt{T\ln d}~.
\]
Indeed, for regularizer~(\ref{eq:f1}), we have
\[
    \left(\sum_{i=1}^d |u_i| \max_t |x_{t,i}|\right)
\le
    \left(\sum_{i=1}^d |u_i| \max_t\max_j |x_{t,j}|\right)
=
    \norm{\bu}_1 \max_t\norm{\bx_t}_{\infty}~.
\]
Similar regularization functions are studied by~\cite{GOB11} although in a different context.

Recently, a framework for studying scale-invariance in online algorithms has been proposed by~\cite{RML13}. In the variant of their setting closest to our model, the sequence of instances $\bx_t\in\R^d$ is such that there exists an unknown diagonal matrix $S$ for which $\norm{S^{1/2}\bx_t}_{\infty} \le 1$ for all $t$. The algorithm they propose is a form of projected gradient descent with a diagonal update (see Subsection~\ref{sec:diagonal} for an explanation of diagonal updates), where adaptivity is achieved by means of a variable learning rate rather than a variable regularizer. Their algorithm achieves a regret bound of the form
\begin{equation}
\label{eq:lang}
    R_T(\bu)
\le
    \frac{C}{2\sqrt{2}}\sum_{i=1}^d\frac{1 + 6\Delta_i + \Delta_i^2}{b_{T,i}}\sqrt{\sum_{t=1}^T \bigl(\ell'_{t,i}\bigr)^2}
\end{equation}
for any $\bu\in\R^d$ such that ${\displaystyle \max_{t=1,\dots,T} \bigl|\bu^{\top}\bx_t\bigr| \le C}$, where $C$ is a parameter used by the algorithm. The quantity $\Delta_i$ is of the form
\[
    \Delta_i = \frac{b_{T,i}}{|x_{t_i,i}|}
\]
where $t_i$ is the first time step where the $i$-th feature has a nonzero value.

Clearly enough, introducing the parameter $C$ in our setting might allow a dynamical tuning of $\eta$ which we could not afford in our analysis. However, a rough comparison can be made by considering the intermediate bound~(\ref{eq:tocompare}) for the regularizer of type~(\ref{eq:f2}). Tuning $\eta = C\sqrt{d/2}$ leads to the regret bound
\[
    R_T(\bu) \le C\sqrt{2} \sum_{i=1}^d \sqrt{L^2+\sum_{t=1}^T \left(\frac{\ell'_{t,i}}{b_{t,i}}\right)^2 }
\]
for any $\bu\in\R^d$ such that
${\displaystyle
     \max_{i=1,\dots,d}\max_{t=1,\dots,T} \bigl(u_i x_{t,i}\bigr)^2 \le C
}$. This last bound now bears some resemblance to~(\ref{eq:lang}), although further study is clearly necessary to bring out the connections between these scale-invariant updates.


\section{Binary classification: aggressive and diagonal updates}
\label{sec-class}
In this section we show the first mistake bounds for Passive-Aggressive~\citep{CrammerDKSSS06} that improve on the standard Perceptron mistake bound, and also prove the first known bound for AROW with diagonal updates. Moreover, we recover ---with some minor improvement--- the known bounds for the second-order Perceptron~\citep{Cesa-BianchiCG05} and non-diagonalized AROW~\citep{CrammerKD09}.

We start by introducing binary classification as a special case of online convex optimization. Let $\fX$ be any finite-dimensional inner product space.
Given a fixed but unknown sequence $(\bx_1,y_1),(\bx_2,y_2),\dots$ of examples $(\bx_t,y_t)\in\fX\times\spin$, let $\ell_t(\bw) = \ell\bigl(\inner{\bw,\bx_t},y_t\bigr)$ be the hinge loss $\bigl[1 - y_t\inner{\bw,\bx_t}\bigr]_+$. It is easy to verify that the hinge loss satisfies the following condition:
\begin{equation}
\text{if $\ell_t(\bw) > 0$ then}
\;
    \ell_t(\bu)
\ge 
    1 + \inner{\bu,\bloss'_t}
\;
\text{for all $\bu,\bw\in\R^d$ with $\bloss_t'\in\partial\ell_t(\bw)$.}
\label{eq:hinge_condition}
\end{equation}
Note that when $\ell_t(\bw) > 0$ the subgradient notation is redundant, as $\partial\ell_t(\bw)$ is the singleton $\bigl\{\nabla\ell_t(\bw)\bigr\}$. In this secton, we apply the OMD algorithm to online binary classification by setting $\bz_t = - \eta_t\bloss_t'$ if $\ell_t(\bw_t) > 0$, and $\bz_t = \bzero$ otherwise.

We prove bounds on the number of steps $t$ in which the algorithm made a prediction mistake, defined by the condition $y_t\,\bw_t^\top\bx_t \le 0$ or, equivalently, by $\ell_t(\bw_t) \geq 1$. In the following, when the number $T$ of prediction steps is understood from the context, we denote by $\scM$ the subset of steps $t$ such that $y_t\,\bw_t^\top\bx_t \le 0$ and by $M$ its cardinality. Similarly, we denote by $\scU$ the set of margin error steps; that is, steps $t$ where $y_t\,\bw_t^\top\bx_t > 0$ and $\ell_t(\bw_t) > 0$. Also, we use $U$ to denote the cardinality of $\scU$. Following a standard terminology, we call \textsl{conservative} or \textsl{passive} an algorithm that updates its classifier only on mistake steps, and \textsl{aggressive} an algorithm that updates its classifier both on mistake and margin-error steps.

\subsection{First-order algorithms}
\label{sec:first_order}
%
%
We start by showing how our framework allows us to generalize and improve previous analyses for binary classification algorithms that use first-order aggressive updates. Let
\[
    L(\bu) = \sum_{t=1}^T \bigl[1 - y_t\inner{\bu,\bx_t}\bigr]_+
\]
be the cumulative hinge loss of $\bu\in\fX$ with respect to some sequence of examples.
The next result provides a general mistake bound for first-order algorithms.
\begin{cor}
\label{cor:gen-aggr}
Assume OMD is run with $f_t=f$, where $f$ has domain $\fX$, is $\beta$-strongly convex with respect to the norm $\norm{\cdot}$, and satisfies $f(\lambda \bu) \le \lambda^2 f(\bu)$ for all $\lambda\in\R$ and all $\bu \in \fX$. Further assume the input sequence is $\bz_t = \eta_t\,y_t\bx_t$ for some $0 \leq \eta_t \le 1$ such that $\eta_t = 1$ whenever $y_t\inner{\bw_t,\bx_t} \le 0$. Then, for all $T \ge 1$,
\begin{align*}
    M
\le
    \argmin_{\bu\in\fX} L(\bu) + D + \frac{2}{\beta}f(\bu) X_T^2 + X_T \sqrt{\frac{2}{\beta}f(\bu)L(\bu)},
\end{align*}
where $M = |\scM|$, ${\displaystyle X_t = \max_{i=1,\dots,t}\norm{\bx_i}_* }$ and
\[
    D = \sum_{t \in \mathcal{U}} \eta_t\left(\frac{\eta_t \norm{\bx_t}_*^2 +2 \beta\, y_t \inner{\bw_t,\bx_t}}{X_t^2} - 2 \right)~.
\]
\end{cor}
\begin{proof}
Fix any $\bu\in\fX$.
Using the second bound of Lemma~\ref{cor:bound_hinge_loss} in the Appendix, with the assumption $\eta_t = 1$ when $t \in \scM$, we get
\begin{align*}
M &\leq L(\bu) + \sqrt{2 f(\bu)} \sqrt{\sum_{t \in \mathcal{M}} \frac{\|\bx_t\|_*^2}{\beta} + \sum_{t \in \mathcal{U}} \left(\frac{\eta_t^2}{\beta} \|\bx_t\|_*^2 +2 \eta_t y_t\inner{\bw_t,\bx_t} \right)} - \sum_{t \in \mathcal{U}} \eta_t \\
&\leq L(\bu) + X_T \sqrt{\frac{2}{\beta} f(\bu)} \sqrt{M + \sum_{t \in \mathcal{U}} \frac{\eta_t^2 \|\bx_t\|_*^2 +2 \beta \eta_t y_t \inner{\bw_t,\bx_t}}{X_t^2}} - \sum_{t \in \mathcal{U}} \eta_t
\end{align*}
where we have used the fact that $X_t \le X_T$ for all $t = 1,\dots,T$. Solving for $M$ we get 
\begin{align}
M \le L(\bu) + \frac{1}{\beta}f(\bu) X_T^2 + X_T \sqrt{\frac{2}{\beta}f(\bu)}\sqrt{\frac{1}{2\beta} X_T^2 f(\bu) + L(\bu) + D'} - \sum_{t \in \mathcal{U}} \eta_t
\label{eq:gen_aggr_bound2}
\end{align}
with $\frac{1}{2\beta} X_T^2 f(\bu) + L(\bu) + D'\geq 0$, and 
\[
    D' = \sum_{t \in \mathcal{U}} \left(\frac{\eta_t^2 \|\bx_t\|_*^2 +2 \beta \eta_t y_t \inner{\bw_t,\bx_t}}{X_t^2} - \eta_t \right)~.
\]
We further upper bound the right-hand side of~(\ref{eq:gen_aggr_bound2}) using the elementary inequality $\sqrt{a+b} \le \sqrt{a} +\tfrac{b}{2\sqrt{a}}$ for all $a>0$ and $b \geq -a$. This gives
\begin{align*}
    M
&\le
    L(\bu) + \frac{1}{\beta}f(\bu) X_T^2 + X_T \sqrt{\frac{2}{\beta}f(\bu)}\sqrt{\frac{1}{2\beta} X_T^2 f(\bu) + L(\bu)} \\
&\quad + \frac{X_T D' \sqrt{\frac{2}{\beta}f(\bu)}}{2\sqrt{\frac{1}{2\beta} X_T^2 f(\bu) + L(\bu)}} - \sum_{t \in \mathcal{U}} \eta_t
\\ &\le
    L(\bu) + \frac{1}{\beta} f(\bu) X_T^2 + X_T \sqrt{\frac{2}{\beta}f(\bu)}\sqrt{\frac{1}{2 \beta} X_T^2 f(\bu) + L(\bu)} + D' - \sum_{t \in \mathcal{U}} \eta_t~.
\end{align*}
Applying the inequality $\sqrt{a+b} \le \sqrt{a} + \sqrt{b}$ and rearranging gives the desired bound.
\end{proof}
The $p$-norm Perceptron of \cite{Gentile03} is obtained by running OMD in conservative mode with $f_t = f = \tfrac{1}{2}\norm{\cdot}_p^2$ for $1 < p \le 2$. In this case we have $\mathcal{U} =
\emptyset$, $\norm{\cdot}_* = \norm{\cdot}_q$ where $q =\tfrac{p}{p-1}$, and $\beta = p-1$ because
$\tfrac{1}{2}\norm{\cdot}^2_p$ is $(p-1)$-strongly convex with respect
to $\norm{\cdot}_p$ for $1 < p \le 2$, see~\cite[Lemma~17]{Shalev-Shwartz07}. Hence Corollary~\ref{cor:gen-aggr} delivers the mistake bound of~\cite{Gentile03}.

However, the term $D$ in the bound of Corollary~\ref{cor:gen-aggr} can also be negative. We can minimize it, subject to $0 \le \eta_t \le 1$, by setting
\[
    \eta_t = \max\left\{\min\left\{\frac{X_t^2 - \beta y_t\inner{\bw_t,\bx_t}}{\|\bx_t\|^2_*},1\right\},0\right\}~.
\]
This tuning of $\eta_t$ is quite similar to that of the Passive-Aggressive algorithm (type I) of~\cite{CrammerDKSSS06}. In fact for $f_t = f = \tfrac{1}{2}\norm{\cdot}_2^2$ we would have 
\[
    \eta_t = \max\left\{\min\left\{\frac{X_t^2 - y_t\inner{\bw_t,\bx_t}}{\|\bx_t\|^2},1\right\},0\right\}
\]
while the update rule for PA-I is
\[
    \eta_t = \max\left\{\min\left\{\frac{1 - y_t\inner{\bw_t,\bx_t}}{\|\bx_t\|^2},1\right\},0\right\}~.
\]
The mistake bound of Corollary~\ref{cor:gen-aggr} is however better than the aggressive bounds for PA-I of~\cite{CrammerDKSSS06} and~\cite{Shalev-Shwartz07}. Indeed, while the PA-I bounds are generally worse than the Perceptron mistake bound
\[
    M
\le
    L(\bu) + \bigl(\norm{\bu}X_T\bigr)^2 + \norm{\bu}X_T \sqrt{L(\bu)}
\]
as discussed by \cite{CrammerDKSSS06}, our bound is better as soon as $D < 0$. Hence, it can be viewed as the first theoretical evidence in support of aggressive updates. It also improves over previous attempts to justify aggressive updates in~\citep{OrabonaKC09,JieOFCC10}.


\subsection{Second-order algorithms}
\label{sec:second_order}
We now move on to the analysis of second-order algorithms for binary classification. Here we use $\fX=\R^d$ and let the inner product $\inner{\bu,\bx}$ be the standard dot product $\bu^{\top}\bx$.

Second-order algorithms for binary classification are online variants of Ridge regression. Recall that the Ridge regression linear predictor is defined by
\[
    \bw_t = \argmin_{\bw\in\R^d}\left(\sum_{s=1}^t\bigl(\bw^{\top}\bx_s - y_s \bigr)^2 + \norm{\bw}^2\right)~.
\]
The closed-form expression for $\bw_t$, involving the design matrix $S_t = \bigl[\bx_1,\dots,\bx_t\bigr]$ and the label vector $\by_t = (y_1,\dots,y_t)$, is given by
$
    \bw_t = \bigl(I + S_t^{\top}S_t\bigr)^{-1}S_t\by_t
$.
The second-order Perceptron (see below) uses this weight $\bw_t$, but $S_t$ and $\by_t$ only contain the examples $(\bx_s,y_s)$ such that $y_s\,\bw_s^{\top}\bx_s \le 0$. Namely, those previous examples on which a mistake occurred. In this sense, it is an online variant of the Ridge regression algorithm.

In practice, second-order algorithms may perform better than their first-order counterparts, such as the algorithms in the Perceptron family. There are two basic second-order algorithms: the second-order Perceptron of~\cite{Cesa-BianchiCG05} and the AROW algorithm of~\cite{CrammerKD09}. We show that both of them are instances of OMD and recover their mistake bounds as special cases of our analysis.

Let $\bz_t = \eta_t\,y_t\bx_t$ and $f_t(\bx)=\tfrac{1}{2} \bx^{\top}A_t\,\bx$, where $A_0=I$ and $A_t=A_{t-1} + \tfrac{1}{r}\bx_t \bx_t^{\top}$ with $r>0$. Each dual function $f^*_t$ is given by $f^*_t(\bx) = \tfrac{1}{2} \bx^{\top} A^{-1}_t \bx$. The functions $f_t$ are $1$-strongly convex with respect to the norm $\normt{\bx}=\sqrt{\bx^{\top}A_t\, \bx}$ with dual norm $\dualnormt{\bx}=\sqrt{\bx^{\top}A_t^{-1}\,\bx}$.

The conservative version of OMD run with $f_t$ chosen as above and $r=1$ corresponds to the second-order Perceptron. The aggressive version corresponds instead to AROW with a minor difference. Let $m_t = \btheta_t^\top A^{-1}_{t-1}\bx_t$. Then for AROW we have $\bw_t^{\top}\bx_t = m_t$ whereas for OMD it holds that $\bw_t^\top\bx_t = \btheta_t^\top A^{-1}_t\bx_t = m_t\frac{r}{r+\chi_t}$, where we used the Woodbury identity and set $\chi_t = \bx^{\top}_t A^{-1}_{t-1} \bx_t$. Note that the sign of $\bw_t^{\top}\bx_t$ is the same for both algorithms, but OMD updates when $y_t m_t\frac{r}{r+\chi_t} \le 1$ while AROW updates when $y_t m_t\leq 1$. Typically, for $t$ large the value of $\chi_t$ is small and the two algorithms behave similarly.

In order to derive a mistake bound for OMD run with $f_t(\bx)=\tfrac{1}{2} \bx^{\top}A_t\,\bx$, first observe that using the Woodbury identity we have
\[
    f^*_t(\theta_t)-f^*_{t-1}(\theta_t)= - \frac{(\bx^\top_t A^{-1}_{t-1} \btheta_t)^2}{2 (r+\bx_t^\top A^{-1}_{t-1} \bx_t)}=-\frac{m_t^2}{2 (r+\chi_t)}~.
\]
Hence, using the second bound of Lemma~\ref{cor:bound_hinge_loss} in the Appendix, and setting $\eta_t=1$, we obtain
\begin{align*}
M &+ U - L(\bu) \\
&\leq \sqrt{\bu^{\top} A_T\, \bu} \sqrt{\sum_{t \in \mathcal{M} \cup \mathcal{U}} \left(\bx^{\top}_t A^{-1}_t \bx_t + 2 y_t \bw_t^\top \bx_t -\frac{m_t^2}{r+\chi_t} \right)} \\
&\leq \sqrt{\|\bu\|^2 + \frac{1}{r}\sum_{t \in \mathcal{M} \cup \mathcal{U}} (\bu^{\top} \bx_t)^2 } \sqrt{r \ln|A_T| + \sum_{t \in \mathcal{M} \cup \mathcal{U}} \left(2 y_t \bw_t^\top \bx_t -\frac{m_t^2}{r+\chi_t} \right)}\\
&= \sqrt{r\,\|\bu\|^2 + \sum_{t \in \mathcal{M} \cup \mathcal{U}} \bigl(\bu^{\top} \bx_t\bigr)^2 } \sqrt{\ln|A_T| + \sum_{t \in \mathcal{M} \cup \mathcal{U}} \frac{m_t (2 r y_t - m_t)}{r (r+\chi_t)} }
\end{align*}
for all $\bu\in\R^d$.

This bound improves slightly over the known bound for AROW in the last sum in the square root. In fact in AROW we have the term $U$, while here we have
\begin{align*}
\sum_{t \in \mathcal{M} \cup \mathcal{U}} \frac{m_t (2 r y_t - m_t)}{r (r+\chi_t)} \leq \sum_{t \in \mathcal{U}} \frac{m_t (2 r y_t - m_t)}{r (r+\chi_t) } \leq 
\sum_{t \in \mathcal{U}} \frac{r^2}{r (r+\chi_t)} \leq U
\end{align*}
where the first inequality holds because $t \in \scM$ implies $y_t m_t \le 0$, which in turn implies $m_t (2 r y_t - m_t) \le 0$.
In the conservative case, when $\mathcal{U} \equiv \emptyset$, the bound specializes to the standard second-order Perceptron bound.

\subsection{Diagonal updates}
\label{sec:diagonal}
Computing $f^*_t$ in AROW and the second-order Perceptron requires inverting $A_t$, which can be done from $A_{t-1}^{-1}$ in time quadratic in $d$. A much better scaling, linear in $d$, can be obtained when the algorithm use a diagonalized version of $A_t$. We now use Corollary~\ref{cor:bound_hinge_loss} to prove a mistake bound for the diagonal version of the second-order Perceptron. Denote $D_t=\diag{A_t}$ be the diagonal matrix that agrees with $A_t$ on the diagonal, where $A_t$ is defined as before, and let $f_t(\bx)=\frac{1}{2} \bx^{\top} D_t\,\bx$. Setting $\eta_t=1$, using the second bound of Lemma~\ref{cor:bound_hinge_loss}, and using also Lemma~\ref{lemma:log_diagonal},
we have\footnote{We did not optimize the constant multiplying $U$ in the bound.}
\begin{align}
&M + U
\le \argmin_{\bu\in\R^d} L(\bu) + \sqrt{\bu^T D_T \bu \left(r \sum_{i=1}^d \ln \left( \frac{1}{r}\sum_{t \in \mathcal{M} \cup \mathcal{U} } x^2_{t,i} + 1\right) + 2U \right) } \nonumber \\
&\quad = \argmin_{\bu\in\R^d} L(\bu) \nonumber \\
&\qquad + \sqrt{\|\bu\|^2+\frac{1}{r}\sum_{i=1}^d u_i^2 \left(\sum_{t\in \mathcal{M} \cup \mathcal{U}} x^2_{t,i}\right)} \sqrt{r \sum_{i=1}^d \ln \left( \frac{1}{r}\sum_{t \in \mathcal{M} \cup \mathcal{U} } x^2_{t,i} + 1\right) + 2U }~. \label{eq:diag-bound}
\end{align}
This allows us to theoretically analyze the cases where this algorithm could be advantageous.
For example, features of textual data are typically binary, and it is often the case that most of the features are zero most of the time. On the other hand, these ``rare'' features are usually the most informative ones ---see, e.g., the discussion of \cite{CrammerDP08a,Crammer:2012:CLC:2343676.2343704}. 

\begin{figure}[t]
\centering
\includegraphics[width=0.5\linewidth]{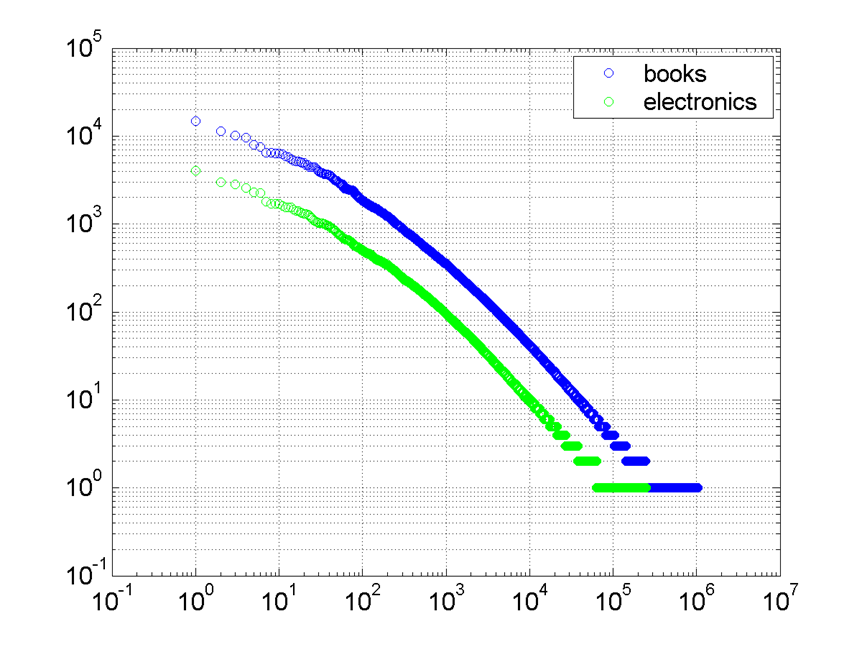}
\caption{Evidence of heavy tails for textual data. The plots show the number of words vs.\ the word rank on two sentiment data sets.}
\label{fig:nlp}
\end{figure}

Figure~\ref{fig:nlp} shows the number of times each feature (word) appears in two sentiment datasets vs.\ the word rank. Clearly, there are a few very frequent words and many rare words. These exact properties originally motivated the CW and AROW algorithms, and now our analysis provides a theoretical justification. Concretely, the above considerations support the assumption that the optimal hyperplane $\bu$ satisfies
\begin{align*}
\sum_{i=1}^d u_i^2 \sum_{t\in \mathcal{M} \cup \mathcal{U}} x^2_{t,i} \approx \sum_{i \in \mathcal{I}} u_i^2 \sum_{t\in \mathcal{M} \cup \mathcal{U}} x^2_{t,i} \leq s \sum_{i \in \mathcal{I}} u_i^2 \approx s \|\bu\|^2
\end{align*}
where $\mathcal{I}$ is the set of informative and rare features, and $s$ is the maximum number of times these features appear in the sequence. 
Running the diagonal version of the second order Pereptron so that $\mathcal{U} = \emptyset$, and assuming that
\begin{align}
\label{eq:nlp}
\sum_{i=1}^d u_i^2 \sum_{t\in \mathcal{M} \cup \mathcal{U}} x^2_{t,i} \leq s \|\bu\|^2
\end{align}
the last term in the mistake bound~(\ref{eq:diag-bound}) can be re-written as
\begin{align*}
    &\sqrt{\|\bu\|^2+\frac{1}{r}\sum_{i=1}^d u_i^2 \sum_{t\in \mathcal{M}} x^2_{t,i}}
    \sqrt{r \sum_{i=1}^d \ln \left( \frac{1}{r}\sum_{t \in \mathcal{M}} x^2_{t,i} + 1\right)} \\
&\quad \leq
    \|\bu\| \sqrt{r + s} \sqrt{ d \ln \left( \frac{M X_T^2}{d r} + 1 \right)}
\end{align*}
where we calculated the maximum of the sum, given the constraint
\[
    \sum_{i=1}^d \sum_{t \in \mathcal{M}} x^2_{t,i} \leq X_T^2 M~.
\]
We can now use Corollary~\ref{corollary:log2} in the Appendix to obtain
\begin{align*}
M &\leq \argmin_{\bu\in\R^d} L(\bu) + \|\bu\| \sqrt{(r+s) d \ln \left(\frac{\sqrt{8} \|\bu\|^2 (r+s) X_T^4}{e d r^2}  + 2 L(\bu)\frac{X_T^2}{d r}+2 \right)}~.
\end{align*}
Hence, when the hypothesis (\ref{eq:nlp}) is verified, the number of mistakes of the diagonal version of AROW depends on $\sqrt{\ln L(\bu)}$ rather than on $\sqrt{L(\bu)}$.

Diagonal updates for online convex optimization were also proposed and analyzed by~\citep{McMahanS10,DuchiHS10,RML13}. When instantiated to the binary classification setting studied in this section, their analysis delivers regret bounds which are not comparable to ours.

\section{Conclusions}
\label{sec-conc}
We proposed a framework for online convex optimization combining online mirror descent with time-varying regularizers. This allowed us to view second-order algorithms (such as the Vovk-Azoury-Warmuth algorithm, the second-order Perceptron, and the AROW algorithm) and algorithms for composite losses as special cases of mirror descent. Our analysis also captures second-order variants that only employ the diagonal elements of the second order information matrix, a result which was not within reach of the previous techniques.

Within our framework, we also derived and analyzed new regularizers based on an adaptive weighted version of the $p$-norm Perceptron. These regularizers generate instances of OMD that are both efficient to implement and invariant to rescaling of individual coordinates in the data. In the case of sparse targets, the corresponding instances of OMD achieve performance bounds better than that of OMD with $1$-norm regularization. 

We also improved previous bounds for existing first-order algorithms. For example, we were able to formally explain the phenomenon according to which aggressive algorithms often exhibit better empirical performance than their conservative counterparts. Specifically, our refined analysis provides a bound for Passive-Aggressive (PA-I) that is never worse (and sometimes better) than the Perceptron bound.

One interesting direction to pursue is the derivation and analysis of algorithms based on time-varying versions of the entropic regularizers used by the EG and Winnow algorithms. A remarkable recent result along these lines is the work of \cite{SteinhardtL14}, in which a time-varying entropic regularizers is used to obtain an improved version of EG for the prediction with experts setting (a special case of online convex optimization).

More in general, it would be useful to devise a more systematic approach to the design of adaptive regularizers enjoying a given set of desired properties, such as invariance to rescalings. This should help in obtaining more examples of adaptation mechanisms that are not based on second-order information.

\section*{Acknowledgements}
The second author gratefully acknowledges partial support by an Israeli Science Foundation grant ISF-1567/10.
The third author acknowledges partial support by MIUR (project ARS TechnoMedia, PRIN 2010-2011, grant no.\ 2010N5K7EB 003).



\appendix
\section*{Technical lemmas}
\begin{proof}[Proof of Lemma~\ref{l:weighted-q}]
Define $h(\bw) = g(A \bw)$, where $A$ is an invertible matrix.
First note that the Fenchel conjugate of $h(\bw)$ is $h^*(\btheta) = g^*(A^{-\top} \btheta)$ ---see for example \citep[Proposition 13.20 IV]{BauschkeC11}. Hence, the Fenchel conjugate of $f$ is obtained by setting: $A = \mathrm{diag}\bigl(\{a_1^{1/q},\dots,a_d^{1/q}\}\bigr)$, $g(\bw)=\frac{1}{2(p-1)}\norm{\btheta}_p^2$, and by using the known Fenchel conjugate of $g$.

In order to show the second part, using \citep[Lemma~14]{Shalev-Shwartz07} we may prove strong convexity of $h$ w.r.t.\ a norm $\norm{\cdot}$ by showing that 
\begin{equation}
  \label{eq:lemmweightpnorm1}
  \bx^\top \nabla^2 h(\bw) \bx = \bx^\top A^\top \nabla^2 g(A \bw) A \bx \geq \norm{\bx}^2~.
\end{equation}
Moreover, \citep[Lemma~17]{Shalev-Shwartz07} proves that, for any $\bx,\bw \in \R^d$, we have
\begin{equation}
\label{eq:lemmweightpnorm2}
\bx \nabla^2 g(\bw) \bx \geq \|\bx\|^2_p~.
\end{equation}
Putting together \eqref{eq:lemmweightpnorm1} and \eqref{eq:lemmweightpnorm2} and the same setting of $A$ and $g$ used above, we have that the strong convexity of $f$.

We now prove that the dual norm of $\left( \sum_{i=1}^d \vert x_i
  \vert^q b_i \right)^{1/q}$ is $\left( \sum_{i=1}^d \vert \theta_i
  \vert^p b_i^{1-p} \right)^{1/p}$. By definition of dual norm, 
\begin{align*}
\sup_{\bx}\theset{\bu^{\top}\bx}{ \left( \sum_{i=1}^d \vert x_i
  \vert^q b_i \right)^{1/q}\le 1} 
&= \sup_{\bx}\theset{\bu^{\top}\bx}{ \left( \sum_{i=1}^d \left\vert x_i
  b_i^{1/q}\right\vert^q \right)^{1/q}\le 1} \\
&= \sup_{\by}\theset{\sum_i u_i y_i b_i^{-1/q}}{ \left( \sum_{i=1}^d \left\vert y_i
  \right\vert^q \right)^{1/q}\le 1} \\
&= \left\Vert \paren{u_1 b_1^{-1/q},\dots,u_d b_d^{-1/q}} \right\Vert_p
\end{align*}
where $1/q+1/p=1$. Writing the last norm explicitly and observing that $p=q/(q-1)$,
\begin{align*}
\paren{\sum_i { \vert u_i \vert ^p b_i^{-p/q}}}^{1/p} = \paren{\sum_i { \vert u_i \vert ^p b_i^{1-q}}}^{1/p}
\end{align*}
which concludes the proof.
\end{proof}
\begin{lemma}
\label{cor:bound_hinge_loss}
Assume OMD is run with functions $f_1,f_2,\dots,f_T$ defined on $\fX$ and such that each $f_t$ is $\beta_t$-strongly convex with respect to the norm $\norm{\cdot}_{f_t}$ and $f_t(\lambda \bu) \le \lambda^2 f_t(\bu)$ for all $\lambda\in\R$ and all $\bu \in S$.
For each $t=1,2,\dots,T$ let $\dualnormt{\cdot}$ be the dual norm of $\normt{\cdot}$.
Assume further the input sequence is $\bz_t = - \eta_t\bloss_t'$ for some $\eta_t > 0$, where $\bloss_t'\in\partial\ell_t(\bw_t)$, $\ell_t(\bw_t) = 0$ implies $\bloss_t' = \bzero$, and $\ell_t = \ell\bigl(\inner{\cdot,\bx_t},y_t\bigr)$ satisfies~(\ref{eq:hinge_condition}). Then, for all $T \ge 1$,
\begin{align}
\label{eq:pre-mb}
&\sum_{t \in \mathcal{M} \cup \mathcal{U}} \eta_t \leq L_{\eta} + \lambda f_T(\bu) + \frac{1}{\lambda} \left(B+\sum_{t \in \mathcal{M} \cup \mathcal{U}} \left(\frac{\eta_t^2}{2 \beta_t} \dualnormt{\bloss'_t}^2 -\eta_t \inner{\bw_t,\bloss'_t}\right)\right)
\end{align}
for any $\bu \in S$ and any $\lambda>0$, where
\[
    L_{\eta} = \sum_{t \in \mathcal{M} \cup \mathcal{U}} \eta_t\,\ell_t(\bu)
\qquad\text{and}\qquad
    B=\sum_{t=1}^{T} \bigl(f^*_t(\btheta_t)-f^*_{t-1}(\btheta_{t})\bigr)~.
\]
In particular, choosing the optimal $\lambda$, we obtain
\begin{align*}
&\sum_{t \in \mathcal{M} \cup \mathcal{U}} \eta_t \leq L_{\eta} + 2 \sqrt{f_T(\bu)} \sqrt{\left[ B+\sum_{t \in \mathcal{M} \cup \mathcal{U}} \left(\frac{\eta_t^2}{2\beta_t} \dualnormt{\bloss'_t}^2 -\eta_t \inner{\bw_t,\bloss'_t}\right)\right]_+}~.
\end{align*}
\end{lemma}
\begin{proof}
We apply Lemma~\ref{lemma:general} with $\bz_t = - \eta_t\bloss_t'$ and using $\bu = \lambda \btildeu$ for any $\lambda > 0$,
\begin{align*}
    \sum_{t=1}^T \eta_t\inner{\bloss'_t,\bw_t-\lambda\btildeu}
\le
    \lambda^2 f_T(\btildeu) + \sum_{t=1}^{T} \left(\frac{\eta_t^2}{2\beta_t}\dualnormt{\bloss'_t}^2 + f^*_t(\btheta_t)-f^*_{t-1}(\btheta_{t}) \right)~.
\end{align*}
Since $\ell_t(\bw_t) = 0$ implies $\bloss_t' = \bzero$, and using~(\ref{eq:hinge_condition}),
\[
    \sum_{t \in \mathcal{M} \cup \mathcal{U}} \Bigl(\eta_t\inner{\bloss'_t,\bw_t} + \eta_t - \eta_t\,\ell_t(\bu) \Bigr)
\le
    \sum_{t=1}^T \eta_t\inner{\bloss'_t,\bw_t-\lambda\btildeu}~.
\]
Dividing by $\lambda$ and rearranging gives the first bound. The second bound is obtained by choosing the $\lambda$ that makes equal the last two terms in the right-hand side of~(\ref{eq:pre-mb}).
\end{proof}
\begin{lemma}
\label{lemma:log_diagonal}
For all $\bx_1,\dots\bx_T\in\R^d$ let $D_t=\diag{A_t}$ where $A_0=I$ and $A_t=A_{t-1} + \tfrac{1}{r}\bx_t \bx_t^{\top}$ for some $r > 0$. Then
\begin{align*}
\sum_{t=1}^T \bx_t D_t^{-1} \bx_t \leq r \sum_{i=1}^d \ln \left( \frac{1}{r}\sum_{t=1}^T \bx^2_{t,i}+1\right)~.
\end{align*}
\end{lemma}
\begin{proof}
Consider the sequence $a_t \geq 0$ and define $v_t = a_0 + \sum_{i=1}^t a_i$ with $a_0 >0$. The concavity of the logarithm implies $\ln b \leq \ln a + \frac{b-a}{a}$ for all $a,b>0$. Hence we have
\begin{align*}
\sum_{t=1}^T \frac{a_t}{v_t} = \sum_{t=1}^T \frac{v_t-v_{t-1}}{v_t} \leq \sum_{t=1}^T \ln \frac{v_t}{v_{t-1}}= \ln \frac{v_T}{v_0} = \ln \frac{a_0 + \sum_{t=1}^T a_t}{a_0}~.
\end{align*}
Using the above and the definition of $D_t$, we obtain
\begin{align*}
\sum_{t=1}^T \bx_t D_t^{-1} \bx_t = \sum_{i=1}^d \sum_{t=1}^T \frac{\bx^2_{t,i}}{1+\frac{1}{r} \sum_{j=1}^t \bx^2_{j,i}} = r \sum_{i=1}^d \sum_{t=1}^T \frac{\bx^2_{t,i}}{r+\sum_{j=1}^t \bx^2_{j,i}}
\leq r \sum_{i=1}^d \ln \frac{r+\sum_{t=1}^T \bx^2_{t,i}}{r}~.
\end{align*}
\end{proof}
We conclude the appendix by proving the results required to solve the implicit logarithmic equations of Section~\ref{sec:diagonal}. We use the following result ---see \citep[Lemma~2]{OrabonaCG12}.
\begin{lemma}
\label{lemma:log1}
Let $a,x>0$ be such that $x \leq a \ln{x}$. Then for all $n > 1$
\[
    x \leq \frac{n}{n-1} a \ln{\frac{n a}{e}}~.
\]
\end{lemma}
This allows to prove the following easy corollaries.
\begin{cor}
\label{corollary:log1}
For all $a,b,c,d,x>0$ such that $x \leq a \ln(b x+c)+d$, we have
\[
    x \leq \frac{n}{n-1}\left(a \ln{\frac{n a b}{e}} + d\right)+\frac{c}{b} \frac{1}{n-1}~.
\]
\end{cor}
\begin{cor}
\label{corollary:log2}
For all $a,b,c,d,x>0$ such that
\begin{equation}
\label{eq:log2_1}
    x \leq \sqrt{a \ln(b x+1)+c}+d
\end{equation}
we have 
\[
    x \leq \sqrt{a \ln\left(\frac{\sqrt{8} a b^2}{e} + 2 b \sqrt{c} + 2 d b +2\right)+c}+d~.
\]
\end{cor}
\begin{proof}
Assumption~(\ref{eq:log2_1}) implies
\begin{align}
x^2 &\leq \left(\sqrt{a \ln(b x+1)+c}+d\right)^2 
\leq 2 a \ln(b x+1) + 2 c + 2 d^2 = a \ln(b x+1)^2 + 2 c + 2 d^2 \nonumber \\
&\leq a \ln(2 b^2 x^2 + 2) + 2 c + 2 d^2~. \label{eq:log2_2} 
\end{align}
From Corollary \ref{corollary:log1} we have that if $f,g,h,i,y>0$ satisfy $y \leq f \ln(g x+h)+i$, then
\[
y \leq \frac{n}{n-1}\left(f \ln{\frac{n f g}{e}} + i\right)+\frac{h}{g} \frac{1}{n-1} \leq \frac{n}{n-1}\left(\frac{n f^2 g}{e^2} + i\right)+\frac{h}{g} \frac{1}{n-1}
\]
where we have used the elementary inequality $\ln y \leq \frac{y}{e}$ for all $y \geq 0$.
Applying the above to~(\ref{eq:log2_2}) we obtain
\begin{align*}
x^2 &\leq \frac{n}{n-1}\left(\frac{2 n a^2 b^2}{e^2} + 2 c + 2 d^2 \right)+\frac{1}{b^2} \frac{1}{n-1}
\end{align*}
which implies
\begin{align}
x &\leq \sqrt{\frac{n}{n-1}}\left(\frac{\sqrt{2 n} a b}{e} + \sqrt{2 c} + \sqrt{2} d \right)+\frac{1}{b} \frac{1}{\sqrt{n-1}}~. \label{eq:log2_3}
\end{align}
Note that we have repeatedly used the elementary inequality $\sqrt{x+y}\leq \sqrt{x}+\sqrt{y}$.
Choosing $n=2$ and applying~(\ref{eq:log2_3}) to~(\ref{eq:log2_1}) we get
\begin{align*}
x \leq \sqrt{a \ln(b x+1)+c}+d
  \leq \sqrt{a \ln\left(\frac{\sqrt{8} a b^2}{e} + 2 b \sqrt{c} + 2 d b +2\right)+c}+d
\end{align*}
concluding the proof.
\end{proof}


\bibliographystyle{plainnat}
\bibliography{learning}

\end{document}